\documentclass[conference]{IEEEtran}
\IEEEoverridecommandlockouts
\usepackage{cite}
\usepackage{booktabs}
\usepackage{amsmath,amssymb,amsfonts}
\usepackage{amsthm,bm}
\usepackage[T1]{fontenc}
\usepackage{url}
\usepackage{graphicx}
\graphicspath{{figures/}}
\usepackage{textcomp}
\usepackage{xcolor}
\usepackage[hidelinks]{hyperref}
\def\BibTeX{{\rm B\kern-.05em{\sc i\kern-.025em b}\kern-.08em
    T\kern-.1667em\lower.7ex\hbox{E}\kern-.125emX}}

\makeatletter
\newcommand{\linebreakand}{%
  \end{@IEEEauthorhalign}
  \hfill\mbox{}\par
  \mbox{}\hfill\begin{@IEEEauthorhalign}
}
\makeatother

\newtheorem{proposition}{Proposition}
\newtheorem{lemma}{Lemma}
\newcommand{\T}{^{\top}}
\newcommand{\Ad}{\mathrm{Ad}}
\newcommand{\ad}{\mathrm{ad}}
\newcommand{\R}{\mathbb{R}}
\newcommand{\cC}{\mathcal{C}}

\begin{document}

\title{Composable Dynamics for Cooperative Humanoids: A RoboCompiler Extension%
\thanks{This work was supported by the Post Docs in Companies (PoDoCo) program in Finland, in collaboration with Tampere University and Sandvik Mining and Construction Oy.}
}

\author{
\IEEEauthorblockN{Mehdi Heydari Shahna}
\IEEEauthorblockA{
\textit{Faculty of Engineering}\\
\textit{and Natural Sciences}\\
\textit{Tampere University}\\
Tampere, Finland\\
}
\and
\IEEEauthorblockN{Jouni Mattila}
\IEEEauthorblockA{
\textit{Faculty of Engineering}\\
\textit{and Natural Sciences}\\
\textit{Tampere University}\\
Tampere, Finland
}}

\maketitle

\begin{abstract}
Physically connected humanoids need consistent dynamics for internal actuator loops, inter-robot forces, and changing ground supports.
We extend RoboCompiler with member-wise dynamics assembly, compiling reduced dynamics from rigid bodies, moving actuator parts, and reflected rotor inertias.
Compiled body Jacobians and analytic constraint curvature yield an explicit Coriolis factor satisfying the power identity for floating-base quasi-velocities.
A two-level reduction retains reusable robot models while composing their supports and hand connections; local inertial updates have rank at most six, and prescribed-acceleration constraint wrenches are affine in actuator efforts.
These interfaces support actuator-aware tracking, coupled force allocation, payload updates, and support-mode switching for five hand-linked humanoids.
Across 264 feasible states in eight humanoid configurations, normalised acceleration discrepancies remain within $1.6\!\times\!10^{-10}$ against a hard-constraint reference and $1.4\!\times\!10^{-10}$ against a compiler-independent reference; MuJoCo simulations achieve 0.73~mm root-mean-square (RMS) pelvis error in the ring dance and complete side steps through 20 support-mode changes.
\end{abstract}

\begin{figure}[h!]
\centering
\includegraphics[width=\columnwidth]{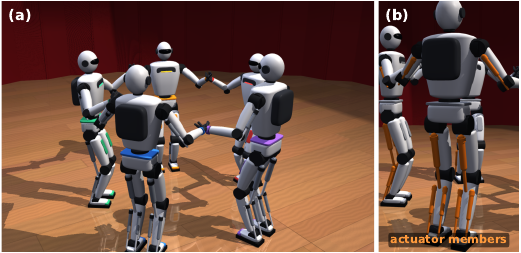}
\caption{(a)~Five closed-chain humanoids dance hand in hand on the realistic MuJoCo plant; from the logged
states. (b)~One humanoid with the moving actuator parts: the screw barrels and rods of knees, ankles and
elbows and the universal-joint crosses of the ankle struts contribute explicitly to the dynamics alongside the links and reflected rotor inertias.}
\label{fig:teaser}
\end{figure}

\begin{figure*}[!t]
\centering
\includegraphics[width=\textwidth]{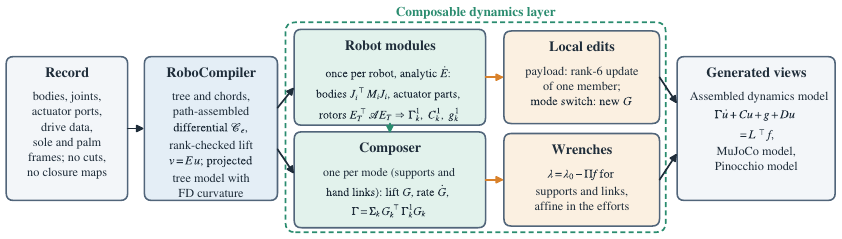}

\caption{RoboCompiler generates the lift $v=Eu$. The extension adds analytic rates and
state-dependent robot modules; the composer $G$ closes supports and hand links. Member edits and prescribed-acceleration
wrench maps serve downstream consumers. MuJoCo and Pinocchio views share the record.}

\label{fig:method}
\end{figure*}

\section{Introduction}

Physically connected humanoids must coordinate whole-body motion and interaction forces while satisfying internal actuator loops and changing ground supports.
Screw-driven knees, parallel ankles, and hand connections create closures at both robot and formation levels (Fig.~\ref{fig:teaser}).
Their dynamics must consistently account for loop closure, actuator transmission, and constraint curvature~\cite{featherstone},
while retaining each robot's inertial structure as payloads and support modes change.
The \emph{Springer Handbook of Robotics} identifies synchronization and physical interaction as central issues in
cooperative multirobot systems~\cite[Sec.~53.10.3]{parker2016multiple}.
Slipping interfaces add contact-mode and frictional-force requirements~\cite{kao2016contact}, motivating a model that separates persistent mechanical connections from changing environmental contacts.
Partial velocities and coordinate reduction establish dynamics in independent speeds and the recovery of constraint
forces~\cite{kane,wehage,ghorbel}. Constraint embedding organises loop-constrained submechanisms for recursive
dynamics~\cite{jain,chignoli}, while body-Jacobian formulations provide explicit Coriolis factors~\cite{garofalo,echeandia}.
Rigid-constraint libraries~\cite{pinocchio,proximal} and regularised simulators~\cite{mujoco} implement complementary
constrained dynamics. The integration task is to expose additive member contributions and reusable robot models directly from a physical record.

RoboCompiler~\cite{robocompiler} generates closure differentials,
rank-checked velocity lifts, power-conjugate actuator maps,
and projected tree dynamics from physical records; its
original implementation evaluates constraint curvature by
finite differences. The member-wise formulation of~\cite{ip} expresses dynamics through additive body contributions and supports local inertial changes when the coordinate-reduction map is supplied. The present extension connects these interfaces:
it compiles the member Jacobians and their analytic rates,
includes moving actuator bodies and reflected rotor inertias,
and separates internal robot reduction from the composition
of supports and inter-robot connections.

Our framework (Fig.~\ref{fig:method}) provides the following
contributions: 1) An algorithm that assembles member-level inertia,
Coriolis, gravity, damping, and actuation terms directly
from compiled physical records, with analytic lift curvature
and a member-wise power identity for floating-base
quasi-velocities. 2) A two-level composition that retains each robot's
internal reduction across support and connection changes,
while exposing local inertial updates and
prescribed-acceleration constraint-wrench maps. 3) A common implementation evaluated through independent
dynamics formulations, computational benchmarks, and
cooperative humanoid simulations with actuator effects,
payload changes, and support switching, supplemented by
state-wise checks on two other robot systems.

The code, results, and videos are publicly available at \textcolor{blue}{\url{https://github.com/Mehdi-Heydari-Shahna/RoboCompiler-Humanoids}}.

\section{Background}\label{sec:background}
Spatial motion vectors are ordered as angular and linear components, $V=[\omega;\nu]\in\R^6$, and spatial forces as
moment and force. Poses $T\in\mathrm{SE}(3)$ map body frames to the world; $\Ad_T$ transforms twists, $\ad_V$ represents
the Lie bracket, and $\hat x y=x\times y$. The operator $\vee$ converts a Lie-algebra matrix to its six-vector, and
$J_\ell$ is the left Jacobian of the exponential map. Dots denote derivatives along the physical motion; $I_d$ is the $d\times d$ identity.
The augmented configuration $q$ collects root poses, tree coordinates, and cut-joint coordinates. Its velocity
$v\in\R^n$ satisfies $\dot q=N(q)v$, where $N$ maps quasi-velocities to configuration rates; a floating root contributes
its body twist. Chord $e$ compares frames $T_e^-$ and $T_e^+$ through
$\rho_e=\log((T_e^-)^{-1}T_e^+)^\vee$. Its closure differential is~\cite{robocompiler,pacdm}
\begin{equation}\label{eq:pacdm}
\cC_e=A_eS\,\mathrm{diag}(\sigma_e),\qquad A_e=J_\ell^{-1}(\rho_e)\Ad_{(T_e^-)^{-1}},
\end{equation}
where $S\in\R^{6\times n}$ contains the world screws and $\sigma_e\in\{-1,0,1\}^n$ is the signed path incidence.
Chords sharing dependent coordinates form a module. With $p$ its dependent columns, $a$ its input columns, and
$\mathsf S$ a fixed selector of independent rows, rank-checked reduction gives
$v_p=-(\mathsf S\cC_p)^{-1}\mathsf S\cC_av_a$. Stacking modules yields
\begin{equation}\label{eq:lift}
v=E(q)u,\qquad \cC(q)E(q)=0.
\end{equation}
Here $E\in\R^{n\times d}$, $u\in\R^d$, $\cC$ stacks chord differentials, and $d$ is the number of independent speeds. For $n_T$ tree speeds, let
$E_T\in\R^{n_T\times d}$ contain the tree rows, $M_T$ and $h_T$ be the tree inertia and bias, and
$B\in\R^{n_T\times n_f}$ map the $n_f$ actuator efforts $f$ to tree forces. RoboCompiler projects these terms as
\begin{equation}\label{eq:rc}
\begin{gathered}
M_r\dot u+h_r=L\T f,\qquad L=B\T E_T,\\
M_r=E_T\T M_TE_T,\quad h_r=E_T\T(h_T+M_T\dot E_Tu).
\end{gathered}
\end{equation}
The product $\dot E_Tu$ is the constraint-induced acceleration bias, evaluated by a centred directional difference in
the original implementation~\cite[Listing~4]{robocompiler}.

\section{Composable Compiled Dynamics}\label{sec:method}
\subsection{Compiled body Jacobians and members}
Let $a_i\in\{0,1\}^n$ mark the coordinates on the tree path from the root to body $i$ and $T_i$ its pose. The body-frame
twist of the body is
\begin{equation}\label{eq:J}
V_i=J_i(q)\,u,\qquad J_i=\Ad_{T_i^{-1}}S\,\mathrm{diag}(a_i)\,E ,
\end{equation}
where $J_i\in\R^{6\times d}$ is the matrix of partial velocities~\cite{kane}. Along a motion,
$\dot J_i=\Ad_{T_i^{-1}}\big[(\dot S-\ad_{V^w_i}S)\mathrm{diag}(a_i)E+S\,\mathrm{diag}(a_i)\dot E\big]$, with $V^w_i$ the
world twist $V_i^w=\Ad_{T_i}V_i$, and $\dot S_j=\ad_{V^w_{c(j)}}S_j$ for the body $c(j)$ moved by coordinate $j$. The
\emph{members} are the rigid bodies, including the moving parts of actuators (screw barrels and rods,
universal-joint crosses), and the rotors: a drive with rotor inertia $I_{m,j}$ behind a transmission of ratio $N_j$ (for a screw of lead
$\ell_j$, $N_j=2\pi/\ell_j$) adds the reflected inertia $\alpha_j=I_{m,j}N_j^2$ to its tree coordinate. We collect
$\mathcal A=\mathrm{diag}(\alpha)$ and $\mathcal D=\mathrm{diag}(\beta)$ on the $n_T$ tree speeds; $\beta_j$ is the viscous coefficient of coordinate $j$, with zero entries where no damping is specified.

\subsection{Member-wise dynamics assembly}
The constant body-frame spatial inertia is $M_i=[I_i,\hat h_i;-\hat h_i,m_iI_3]$, with mass $m_i$, first moment $h_i=m_ic_i$, centre of mass $c_i$, inertia $I_i$ about the body origin, and identity $I_3$. Write its momentum as $P_i=M_iV_i=[n_i;l_i]$, with angular and linear momenta $n_i,l_i$ and let $\tilde P(P)=\left[\begin{smallmatrix}\hat n&\hat l\\
\hat l&0\end{smallmatrix}\right]$, so that $\tilde P(P)Y=\ad_Y\T P$ for every twist $Y$ and the Newton--Euler equation of
the body in its own frame reads $M_i\dot V_i-\tilde P(P_i)V_i=\mathcal F_i$, with $\mathcal F_i$ the total wrench on it.

\begin{proposition}[Member-wise assembly on the compiled lift]\label{prop:partition}
Assume ideal joints and closures, rotors that act as scalar inertias on tree coordinates, and a closure mode in which
every module differential has the rank of its dependent set. Then the equations of motion in the independent speeds,
with $\tau_{\rm ext}$ the projected external forces, $R_i$ the orientation of body $i$ and $g_0\in\R^3$ the world gravitational
acceleration, are
\begin{equation}\label{eq:eom}
\Gamma\dot u+C\,u+g+D\,u=L\T f+\tau_{\rm ext},
\end{equation}
with the member-wise sums
\begin{align}
\Gamma&=\textstyle\sum_i\Gamma_i+E_T\T\mathcal AE_T,&\Gamma_i&=J_i\T M_iJ_i,\label{eq:Gamma}\\
C&=\textstyle\sum_iC_i+E_T\T\mathcal A\dot E_T,&C_i&=J_i\T\big(M_i\dot J_i-\tilde P(M_iV_i)J_i\big),\label{eq:C}\\
g&=\textstyle\sum_ig_i,&g_i&=-J_i\T M_i\left[\begin{smallmatrix}0\\R_i\T g_0\end{smallmatrix}\right],\label{eq:g}
\end{align}
and $D=E_T\T\mathcal DE_T$. The model coincides with the projected model~\eqref{eq:rc} when $M_T$ contains the rotor
inertias and $h_T$ the viscous forces: $\Gamma=M_r$ and $Cu+g+Du=h_r$.
\end{proposition}
\begin{proof}
Ideal joints and closures do no work on velocities $v=Eu$, so by virtual power
$\sum_iJ_i\T(M_i\dot V_i-\tilde P(P_i)V_i-\mathcal F^g_i)+E_T\T(\mathcal A\dot v_T+\mathcal Dv_T)=L\T f+\tau_{\rm ext}$, with
$\mathcal F^g_i$ the weight. Substituting $\dot V_i=J_i\dot u+\dot J_iu$ and $\dot v_T=E_T\dot u+\dot E_Tu$
gives~\eqref{eq:eom}--\eqref{eq:g}. Cut coordinates carry no inertia, so $J_i=J_i^{\rm tree}E_T$ with the tree Jacobians,
hence $\Gamma=E_T\T M_TE_T$ and $Cu=E_T\T(C_Tv_T+M_T\dot E_Tu)$, where $C_T$ is the tree Coriolis factor. The argument applies directly to quasi-velocities.
\end{proof}

Every body satisfies $\dot\Gamma_i=C_i+C_i\T$, since
$C_i+C_i\T-\dot\Gamma_i=-J_i\T(\tilde P+\tilde P\T)J_i=0$. Reflected rotor contributions satisfy the same identity,
so any sum of member contributions obeys
\begin{equation}\label{eq:power}
\begin{gathered}
u\T(\tfrac12\dot\Gamma-C)u=0,\qquad \mathcal T=\tfrac12u\T\Gamma u,\\
\dot{\mathcal T}=u\T(L\T f+\tau_{\rm ext}-g-Du).
\end{gathered}
\end{equation}
Thus velocity-product forces balance the change of kinetic metric. The factor $C=C(q,u)$ acts on an auxiliary
speed $u_r\in\R^d$, as required by reference-velocity control. For floating-base quasi-velocities,~\eqref{eq:C} includes
the Boltzmann--Hamel terms~\cite{greenwood}. Linearity in member inertias gives the regressor
\begin{equation}\label{eq:regressor}
\Gamma\dot u+Cu_r+g=\sum_iY_i\pi_i+E_T\T\mathcal A(E_T\dot u+\dot E_Tu_r),
\end{equation}
where $Y_i(q,u,\dot u,u_r)\in\R^{d\times10}$ and $\pi_i\in\R^{10}$ collects the mass, three first moments, and six independent
entries of $I_i$~\cite{slotine}.

\subsection{Analytic curvature}
\begin{lemma}[Rate of the lift]\label{lem:curvature}
Under the rank assumption of Proposition~\ref{prop:partition}, along a motion on the closure manifold,
$\dot\cC_e=A_e\big(\dot S-\ad_{V^w_{e^-}}S\big)\mathrm{diag}(\sigma_e)$, with $V^w_{e^-}$ the world twist of the minus
frame of the chord, and for a module with selected rows $\mathsf S$, dependent columns $p$ and inputs $a$,
\begin{equation}\label{eq:Edot}
\dot E_p=-(\mathsf S\cC_p)^{-1}\mathsf S\big(\dot\cC_pE_p+\dot\cC_aE_a+\cC_a\dot E_a\big).
\end{equation}
\end{lemma}
\begin{proof}
On the manifold $\rho_e=0$ and, by the rank assumption, $\dot\rho_e=\cC_ev=0$ on all six rows, so $J_\ell^{-1}$ has no
rate, and $\tfrac{d}{dt}\Ad_{T^{-1}}=-\Ad_{T^{-1}}\ad_{V^w}$. Differentiating $\mathsf S(\cC_pE_p+\cC_aE_a)=0$
gives~\eqref{eq:Edot}; modules are processed in the order of their dependencies.
\end{proof}
The rate preserves the differential's incidence pattern and stacked evaluation, providing curvature without finite differences.

\subsection{Two-level composition}
Consider $K$ compiled mechanisms (here: robots), each with its own chords. Reducing mechanism $k$ over
its internal chords gives a lift $v_k=E^1_kw_k$ to its \emph{level-1 speeds} $w_k$ and, by
Proposition~\ref{prop:partition}, a \emph{robot module} $(\Gamma^1_k,C^1_k,g^1_k,D^1_k,L^1_k)$. A support or an
inter-robot link compares two body frames, so its differential in the level-1 speeds is a difference of compiled world
Jacobians, $\bar\cC_e=A_e(J^{1,w}_{e^+}-J^{1,w}_{e^-})$, plus the columns of the three rotation (chart) coordinates of a
ball-jointed link. Selecting rows as before gives
the composer lift $w=G(q)\,u$, with $w=\mathrm{col}(w_1,\ldots,w_K,w_\chi)$ stacking the level-1 speeds and $n_\chi$ massless link-chart rates. Thus $G\in\R^{(\sum_kd_{1,k}+n_\chi)\times d}$, where $d_{1,k}=\dim w_k$.

\begin{proposition}[Composition]\label{prop:composition}
Let the dependent coordinates of the composer be level-1 speeds or chart coordinates of links, and let $G_k$ be the rows of
$G$ that belong to robot $k$. Then
\begin{equation}\label{eq:compose}
\begin{aligned}
\Gamma&=\textstyle\sum_kG_k\T\Gamma^1_kG_k, & C&=\textstyle\sum_kG_k\T\big(C^1_kG_k+\Gamma^1_k\dot G_k\big),\\
g&=\textstyle\sum_kG_k\T g^1_k, & L&=\big[L^1_kG_k\big]_k,
\end{aligned}
\end{equation}
and $D=\sum_kG_k\T D^1_kG_k$. This is the model of Proposition~\ref{prop:partition} for the one-level lift
$E=\mathrm{blkdiag}(E^1_1,\dots,E^1_K,I)\,G$. At a fixed robot state, its module terms are independent of the active supports and inter-robot links. Here $[\cdot]_k$ denotes vertical stacking, and $I=I_{n_\chi}$ acts on chart speeds with zero inertia, bias, and actuation.
\end{proposition}
\begin{proof}
Nested lifts compose, $J_i=J^1_iG_k$ and $\dot J_i=\dot J^1_iG_k+J^1_i\dot G_k$; substitute
in~\eqref{eq:Gamma}--\eqref{eq:g}.
\end{proof}

Composition follows the covariance of projected dynamics~\cite[Eq.~(22)]{robocompiler} under a rectangular change of
speeds. A support or link mode replaces $G$ and its rate while retaining the internal reductions and robot terms evaluated at the current state.

\subsection{Local edits and constraint wrenches}
\begin{proposition}[Inertial edit]\label{prop:update}
At a fixed state, if the spatial inertia of body $i$ changes by $\Delta M$ (payload, tool, component variant), then
$\Gamma'=\Gamma+J_i\T\Delta MJ_i$, $g'=g-J_i\T\Delta M[0;R_i\T g_0]$ and
$C'=C+J_i\T(\Delta M\dot J_i-\tilde P(\Delta MV_i)J_i)$; the lift, the port map and all other members are unchanged, the
update of $\Gamma$ has rank $r\le6$, and if $\Gamma$ and $\Gamma'$ are nonsingular, their inverses are related by an $r\times r$ Woodbury solve.
\end{proposition}
\noindent The statement is immediate from~\eqref{eq:Gamma}--\eqref{eq:g}, and $\Delta M$ enters every term linearly.

\begin{proposition}[Wrench of a chord]\label{prop:wrench}
Let a model with speeds $w$ and terms $(\Gamma^1,C^1,g^1,D^1,L^1)$ be constrained further by chords with differential
$\bar\cC$, let a module of them have dependent coordinates $p$ and selected rows $\mathsf S$, and let no other module take a
coordinate of $p$ as an input. At fixed $q,w$, let the prescribed acceleration $\dot w$ satisfy
the active constraint acceleration equations, and let $f$ satisfy
the corresponding projected dynamics with external force
$\tau_{\mathrm{ext}}^1$. Then the multipliers on the selected
rows are
\begin{equation}\label{eq:wrench}
\lambda=(\mathsf S\bar\cC_p)^{-\top}\big[\Gamma^1\dot w+C^1w+g^1+D^1w-L^{1\top}f-\tau^1_{\rm ext}\big]_p ,
\end{equation}
Thus $\lambda=\lambda_0-\Pi f$, where $\lambda_0$ is~\eqref{eq:wrench} evaluated at $f=0$ and $\Pi=(\mathsf S\bar\cC_p)^{-\top}[L^{1\top}]_p$; $[\cdot]_p$ selects dependent entries.
\end{proposition}
\begin{proof}
The unreduced equations are $\Gamma^1\dot w+C^1w+g^1+D^1w=L^{1\top}f+\tau^1_{\rm ext}+(\mathsf S\bar\cC)\T\lambda$; by hypothesis the rows $p$
contain only the multipliers of this module, $[\cdot]_p=(\mathsf S\bar\cC_p)\T\lambda$.
\end{proof}
Constraint reactions follow from the dependent-coordinate equations~\cite{wehage}. For a support,~\eqref{eq:wrench} gives the sole reaction; for a hand link it gives the inter-robot force. At the prescribed acceleration, support-wrench limits become linear inequalities in $f$.

\section{Compiler Integration and Humanoid Benchmark}\label{sec:implementation}
\subsection{Implementation}
A physical record lists bodies with inertias, joints with both attachment frames, actuator ports with
effort bounds, drive data (rotor inertia, transmission, damping, friction) and named sole and palm frames; no joint is
marked as a cut and no closure map is supplied. Our implementation of the compiler of~\cite{robocompiler} selects the tree, keeping requested
independent joints, actuated joints and joints with drive data in it, forms the dependency modules and acquires a
feasible configuration with the PACDM solver~\cite{pacdm} of the RoboCompiler repository, used unchanged. During reference generation, hand-link dependent coordinates
are reselected when the corresponding arm Jacobian becomes
poorly conditioned. The resulting changes of representation
follow the covariance relation
in~\cite[Eq.~(22)]{robocompiler}; execution uses the prepared
coordinate selections of the active support mode.

Poses, screws and their rates are computed level by level on stacked arrays, and structurally identical
modules of one layer are solved as one stacked system. One evaluation returns $(\Gamma,C,g,D,L)$, the robot modules, the
composer lift and its rate. The Python/NumPy extension adds routines for analytic rates, member-wise dynamics assembly, composition, inertial updates, and wrench recovery in \texttt{ipd/}, with record-based backend exporters and reproducible benchmark workflows.
The executable model and benchmark definitions specify the
geometry, inertias, drive parameters, contact settings, and
reference motions. Dependency versions are pinned, and result
summaries record package versions and a source-code hash. MuJoCo policy comparisons use the common sensor-noise seed; reported tracking errors
are trajectory RMS values for this fixed seed.
The same compiled arrays generate (i)~the assembled dynamics model, (ii)~an MJCF model in which tree joints carry
their armature, damping and friction, chords are equality constraints and soles touch the floor through native contact,
and (iii)~a Pinocchio model with rigid constraint models for all chords. The views share the physical record and use distinct dynamics algorithms.
The first reference solves a hard-constraint Karush--Kuhn--Tucker (KKT) system in 50-digit ball
arithmetic~\cite{arb}, assembled from double-precision Pinocchio tree terms and physical constraint rows. For tree
accelerations $a_T=E_T\dot u+\dot E_Tu$ and reference $a_T^\star$, the numerical discrepancies are
\begin{equation}\label{eq:metrics}
\begin{aligned}
\delta_a&=\frac{\|a_T-a_T^\star\|_\infty}{\max(1,\|a_T^\star\|_\infty)},\\
\delta_f&=\frac{\|M_T(a_T-a_T^\star)\|_\infty}{\max(1,\|Bf\|_\infty,\|h_T\|_\infty)}.
\end{aligned}
\end{equation}
The second reference assembles maximal-coordinate Newton--Euler equations directly from the record, treating each body
as a free rigid body and each joint through its attachment frames. It bypasses the compiler's tree and lift, sharing
the sampled state and efforts. Its discrepancy is
$\delta_b=\|x-x^\star\|_\infty/\max(1,\|x^\star\|_\infty)$, where $x$ stacks world angular and centre-of-mass
accelerations. All metrics use component values in SI units, following~\cite[Eq.~(23)]{robocompiler}.

\subsection{Humanoid models and formations}\label{sec:systems}
Table~\ref{tab:systems} lists the humanoid closure modes. The synthetic \emph{humanoid} (Fig.~\ref{fig:teaser}) has 42 bodies, 49 joints and 8
independent loops: per side, a knee driven by a ball screw through a planar triangle, an ankle with two passive axes
driven by two screw struts (universal joint on the shank, ball joint on the foot), an elbow driven by a screw, and rotary
hip, shoulder and waist drives, 21 ports in all. The drives carry representative data of electric humanoids
(rotor inertia, gear ratio or screw lead, peak effort, no-load speed, viscous and Coulomb friction); a 5~mm lead gives the
knee screw an equivalent mass of 32~kg along its stroke, more than any link of the leg, making actuator inertia dynamically significant. Five humanoids of different link dimensions and inertias are placed in a ring, a semicircle or a line.
Hand holding is modelled by ideal spherical (ball) joints between rigid hand links: coincident palm points constrain
relative translation while allowing relative rotation, and transmit forces without joint moments about the connection point.
Writing the two world attachment positions as $p_e^-$ and $p_e^+$, the hand constraint is
\begin{equation}\label{eq:hand}
p_e^+(q)-p_e^-(q)=0.
\end{equation}
These persistent connections are composed with the active sole supports; with both soles of every robot as supports, the ring has 210 bodies, 55
chords (robot loops, soles and hand links) and 60 independent speeds.

\begin{table}[!t]
\centering
\caption{Humanoid closure modes: massive bodies, record joints, chords (loops, supports, hand links), actuator ports, independent speeds $d$ and level-1 speeds per robot $d_1$.}
\label{tab:systems}
\footnotesize\setlength{\tabcolsep}{3.2pt}
\begin{tabular}{@{}lrrrrrl@{}}
\toprule
System, mode & Bodies & Joints & Chords & Ports & $d$ & $d_1$ \\
\midrule
Humanoid, no support & 42 & 49 & 8 & 21 & 27 & -- \\
Humanoid, one sole & 42 & 49 & 9 & 21 & 21 & -- \\
Humanoid, two soles & 42 & 49 & 10 & 21 & 15 & -- \\
Ring of five, all soles & 210 & 245 & 55 & 105 & 60 & 5$\times$27 \\
Ring, left soles & 210 & 245 & 50 & 105 & 90 & 5$\times$27 \\
Ring, actuator space & 210 & 245 & 55 & 105 & 60 & 5$\times$27 \\
Semicircle of five & 210 & 245 & 54 & 105 & 63 & 5$\times$27 \\
Line of five & 210 & 245 & 54 & 105 & 63 & 5$\times$27 \\
\bottomrule
\end{tabular}
\end{table}

\subsection{Closed-loop setup}\label{sec:setup}
The \emph{realistic plant} denotes the MuJoCo simulation~\cite{mujoco} generated from the same records and stepped at
1000~Hz, with regularised equality constraints for loops and hand links and unilateral frictional contact under
the soles (rubber on wood: friction coefficient 0.8, torsional friction, elliptic cones,
5~ms contact time constant). Foot--ground contacts can slip or separate, while hand links retain the spherical connection topology. The plant also includes the following drive and sensing effects: Coulomb friction in every
drive and passive bearing (rod-end pins and ball joints are ideal), a first-order lag of 2~ms of every applied
effort, the torque--speed limit of every motor (full effort up to 50\% of the no-load speed), and
zero-mean Gaussian noise on the measured state (standard deviations 0.1~mrad and 10~mrad/s
for rotary joints, 5~$\mu$m and 1~mm/s for screws, 0.25~mm and
0.25~mrad for the base pose, 5~mm/s and 5~mrad/s for its twist).
The \emph{ideal MuJoCo plant} disables joint Coulomb friction, actuator lag, speed-dependent effort limits, and sensor noise, while retaining viscous damping, regularised constraints, and frictional sole contacts; the \emph{Pinocchio plant} integrates the rigid constrained dynamics~\cite{pinocchio,proximal} with ideal sole
welds and Baumgarte position and velocity gains of 400~s$^{-2}$ and 40~s$^{-1}$.
Noise is sampled independently at each control update;
root poses are perturbed through the $SE(3)$ exponential map.
Cut-joint chart coordinates are reconstructed from the
measured body poses. The controller selects the active
mode's independent measured speeds and reconstructs
dependent model speeds through the compiled lifts.
In MuJoCo executions, ideal support anchors are set to the
current measured sole poses at each evaluation, while slip
and separation are governed by the plant's unilateral
contacts. Reference generation applies closure correction; online evaluation
uses a nominal constrained model at the measured configuration,
with dependent speeds reconstructed through the compiled lifts.
The analytic identities hold under the stated rank and closure
assumptions.
The controller updates at 2~ms intervals of simulated time,
while MuJoCo advances in 1~ms steps.
Let $e$ be the independent-coordinate error, $u_d$ the desired speed expressed in the current frame,
$u_r=u_d-\Lambda e$, $\dot u_r=\dot u_d-\Lambda\dot e$, and $s=u-u_r$. The shared tracking law is
\begin{equation}\label{eq:law}
Q=\Gamma(\dot u_r-k_ss)+Cu_r+g+Du,
\end{equation}
where $Q\in\R^d$ is the commanded generalised force, $\Lambda=15I_d$~s$^{-1}$, and $k_s=45$~s$^{-1}$.
Scalar coordinates use their configuration difference. For root $b$, let $T_b$ and $T_{d,b}$ be the actual and desired poses and define the six-dimensional blocks
$e_b=\log(T_{d,b}^{-1}T_b)^\vee$ and $u_{d,b}=X_b\bar v_{d,b}$, with desired root twist $\bar v_{d,b}$,
actual twist $v_b$, and $X_b=\Ad_{T_b^{-1}T_{d,b}}$. Then
$\dot e_b=J_\ell^{-1}(-e_b)(v_b-u_{d,b})$ and
$\dot u_{d,b}=X_b\dot{\bar v}_{d,b}-\ad_{v_b}u_{d,b}$; independent entries form $e$ and $u_d$.
Efforts are allocated through
\begin{equation}\label{eq:allocation}
\begin{aligned}
\min_{f,z}\;&f\T Wf+10^4\|z\|^2+\eta\sum_b\|H\lambda_b(f)\|^2,\\
\text{s.t. }&L\T f+z=Q,\quad f_{\min}\le f\le f_{\max},\\
&\lambda_b(f)\in\mathcal K_b\quad\text{for each support }b.
\end{aligned}
\end{equation}
Here $f_{\min},f_{\max}\in\R^{n_f}$ are the fixed actuator-effort bounds specified by the physical record, $z\in\R^d$ is slack, $W_{jj}=1/\max(|f_{\min,j}|,|f_{\max,j}|)^2$ normalises effort, and $\mathcal K_b$ is the
sole contact-wrench cone~\cite{caron}. For the sole wrench $\lambda_b=[m_b;F_b]$ with moment $m_b$ and force $F_b$ in its contact frame, $H$ selects tangential forces and all moments, dividing moments by 0.1~m;
$\eta=0$ for dances and $2\!\times\!10^{-4}$ for stepping. 
At each evaluated state, the nominal acceleration induced by the
allocated efforts is
\[
\dot u_{\mathrm{alloc}}(f)
=\Gamma^{-1}(L^\top f-Cu-g-Du),
\qquad
\dot w_{\mathrm{alloc}}(f)
=G\dot u_{\mathrm{alloc}}(f)+\dot G u.
\]
Substituting $\dot w_{\mathrm{alloc}}(f)$ into \eqref{eq:wrench} defines the
support-wrench maps $\lambda_b(f)$ used in the objective and
constraints of \eqref{eq:allocation}. At a fixed state, these maps are affine in
$f$, so the allocation remains a quadratic program. Support and
inter-robot wrench predictions use the same allocated nominal
acceleration. The equality $L^\top f+z=Q$ defines the allocation
residual $z$. DAQP~\cite{daqp} solves the quadratic program; reported
allocation slack is $\|z\|_\infty$.
The support limits impose a centre of pressure within 75\% of each sole half-extent, a friction pyramid inscribed in 75\% of the
floor coefficient, and at least 15~N per loaded sole. During stepping, the unloading sole's force cap falls to zero
within 0.35~s before lift-off and rises over 0.30~s starting 0.04~s after reference landing; its friction and pressure limits shrink with the
cap. On the realistic plant, all policies add $0.85F_{c,j}\tanh(v_{d,j}/\varepsilon_j)$ to each drive effort and clip
to the effort bounds, with nominal Coulomb friction $F_{c,j}$, desired drive speed $v_{d,j}$, and smoothing speed
$\varepsilon_j=0.05$~rad/s for rotary drives or 0.005~m/s for screws. The plant applies actuator lag and speed-dependent saturation to the resulting commands. Shared references, gains, and
allocation isolate the effects of the dynamics representation; all model-based policies except gravity+PD use~\eqref{eq:law}.

\emph{Policies.} \emph{Proposed}: the assembled dynamics model of the complete mode. \emph{Projected tree}: RoboCompiler's
model~\eqref{eq:rc} uses Pinocchio's tree inertia (composite rigid-body algorithm, CRBA), Coriolis matrix, and gravity, with reflected rotor inertias added. These terms are projected through the same lift with analytic curvature; allocation and wrench recovery match the proposed model.
\emph{Per-robot}: each robot uses its own compiled model; links are not modelled. \emph{No curvature}: $\dot E$ omitted.
\emph{No actuator inertia}: screw barrels, rods and crosses and the reflected inertia of the screws are removed from
$\Gamma$ and $C$ (their weight stays in $g$). \emph{Massless actuators}: also their weight removed. \emph{Gravity+PD}:
compiled gravity with constant diagonal gains in the same coordinates, $Q=g-\bar\Gamma\,(k_s\Lambda e+(k_sI_d+\Lambda)(u-u_d))$
with $\bar\Gamma$ the diagonal matrix formed from $\Gamma$ at the initial reference, and quasi-static wrench maps in the allocation.
\emph{Actuator PD}: hand-tuned proportional--derivative (PD) loops at the ports, without a model.

For $N_s$ logged samples and $N_p$ tracked points, the position RMS is
\begin{equation}\label{eq:rms}
\epsilon_p=\sqrt{\frac{1}{N_sN_p}\sum_{n=1}^{N_s}\sum_{j=1}^{N_p}
\|p_j(t_n)-p_{d,j}(t_n)\|_2^2},
\end{equation}
where $p_j$ and $p_{d,j}$ are actual and desired world positions, with $N_p=5$ pelvises or $10$ palms; samples are logged at 100~Hz.
Force errors use the analogous RMS of predicted-minus-measured hand-force magnitudes or sole normal forces after 0.5~s.

Each formation dances for 30~s (Fig.~\ref{fig:steps}): a knee-bounce canon travelling around the formation (0.75~s
period), a wave of the linked hands (0.30~m rise, 1.25~s), a common lean, a sideways sway with pelvis and waist twist, and a
bow. In the travelling formations the robots step sideways together. References are generated on
the closure manifold of the mode (worst closure residual $1.0\!\times\!10^{-11}$). At sampled reference states, inverse-dynamics allocation with the composed model satisfies the effort bounds and support limits for each dance; with
per-robot models the allocation residual satisfies $\|z\|_\infty>10^{-3}$
at less than 1\% of ring-dance samples and 31\% of
travelling-ring samples, which leans on the hand links.

\begin{figure}[t]
\centering
\includegraphics[width=\columnwidth]{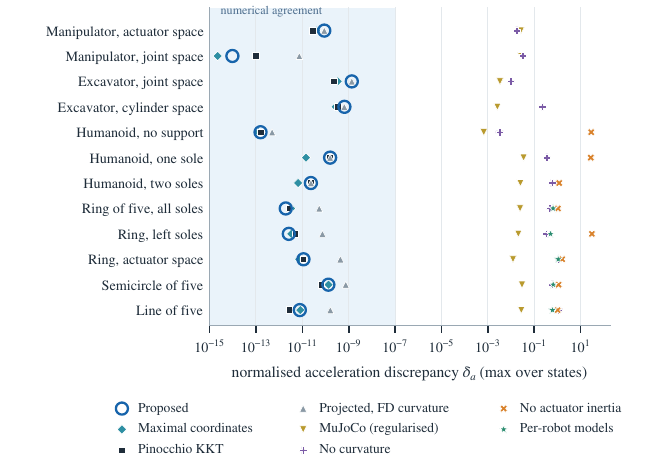}
\caption{State-wise dynamics consistency across humanoid configurations and two additional mechanisms. Humanoid state-wise accelerations against the 50-digit hard-constraint reference ($\delta_a$, max over the
states of each mode). Maximal coordinates: normalised body-acceleration discrepancy $\delta_b$ against the compiler-independent solve. Proposed: member-wise assembly with two-level composition. No curvature, no actuator inertia and per-robot
models remove one ingredient of the model.}
\label{fig:consistency}
\end{figure}

\begin{figure*}[!t]
\centering
\includegraphics[width=\textwidth]{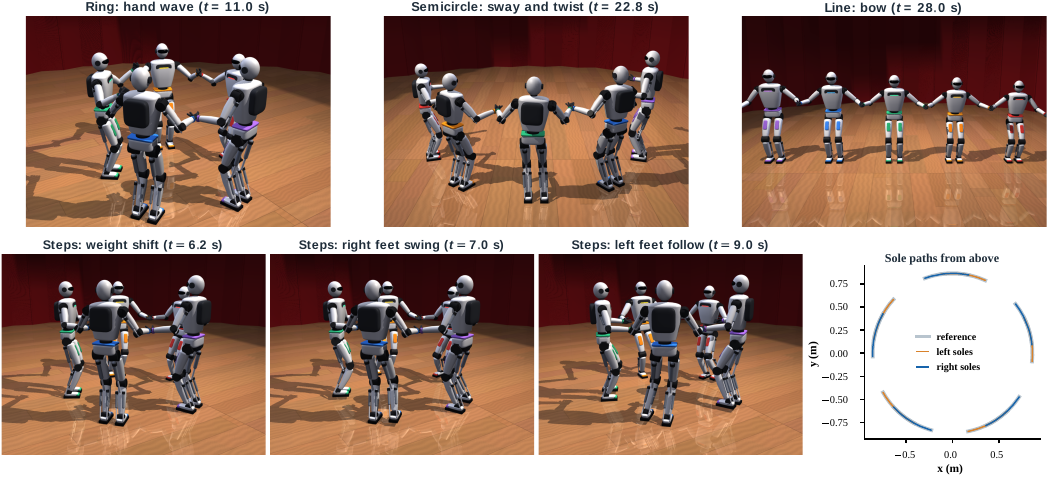}
\caption{Executed motions on the realistic plant, rendered from the logged states. Top: ring (hand wave), semicircle
(sway and twist) and line (bow). Bottom: travelling ring, one side-step cycle, and the sole trajectories of all ten feet
seen from above: grey curves show the references and coloured curves the executed left/right sole paths.}
\label{fig:steps}
\end{figure*}

\begin{table}[!t]
\centering
\caption{Humanoid evaluation times in ms. Member-wise assembly with two-level composition, including $C$ and analytic curvature. Projected:~\eqref{eq:rc} with analytic or finite-difference (FD) curvature. Pinocchio: proximal constrained dynamics~\cite{proximal} (tolerance $10^{-13}$, regularisation $\mu=10^{-9}$, 200-iteration limit; pin chords use two point pairs).}
\label{tab:cost}
\footnotesize\setlength{\tabcolsep}{3.0pt}
\begin{tabular}{@{}lrrrrr@{}}
\toprule
& $d$ & Proposed & \multicolumn{2}{c}{Projected} & Pinocchio \\
\cmidrule(lr){4-5}
System & & & analytic & FD & KKT \\
\midrule
Humanoid & 15 & 1.99 & 1.71 & 2.97 & 0.88 \\
Ring of five & 60 & 3.46 & 5.11 & 9.19 & 18.4 \\
\midrule
\multicolumn{6}{@{}l}{Ring, new closure mode (build and prepare): 15.6} \\
\multicolumn{6}{@{}l}{Ring, payload edit (rank 6): 0.035} \\
\bottomrule
\end{tabular}
\end{table}

\section{Results}\label{sec:results}
Fig.~\ref{fig:consistency} compares accelerations at 264 feasible states (seed 2026): 48 per single-humanoid mode and 24 per five-robot configuration. Independent coordinates are perturbed uniformly about the assembled configuration and corrected to closure; failed corrections and joint-limit-margin violations are rejected. Case-dependent sampling ranges are specified in the executable benchmarks.
The proposed model agrees with the 50-digit hard-constraint reference within $1.6\!\times\!10^{-10}$ in $\delta_a$ and $1.4\!\times\!10^{-10}$ in force-equivalent $\delta_f$, and with the maximal-coordinate reference within $1.4\!\times\!10^{-10}$ in $\delta_b$. Pinocchio's constrained dynamics~\cite{pinocchio,proximal} agree with the reference within $1.6\!\times\!10^{-10}$; one-level and two-level evaluation agree in all terms to $4.9\!\times\!10^{-14}$ (relative). RoboCompiler's projected model reproduces the accelerations using the analytic curvature of Lemma~\ref{lem:curvature}; finite-difference curvature is less accurate (ring: $5.4\!\times\!10^{-11}$ versus $1.9\!\times\!10^{-12}$) and slower (Table~\ref{tab:cost}). Ablation errors reach $1.2\!\times\!10^{0}$ without curvature, $3.1\!\times\!10^{1}$ without actuator-member inertia, and $1.1\!\times\!10^{0}$ without formation links.
Across the sampled humanoid states, MuJoCo's regularised constraints differ from the rigid reference by $1.5\!\times\!10^{-4}$ to $3.5\!\times\!10^{-2}$. The power identity holds within the accuracy of finite-differencing $\Gamma$ along the motion (relative to its largest entry: $3.9\!\times\!10^{-7}$ for the composed model and $6.1\!\times\!10^{-8}$ for single bodies). The regressor reproduces the terms to $2.6\!\times\!10^{-15}$, two partitions of the same ring agree in acceleration to $2.2\!\times\!10^{-13}$, and the compiled ODE's energy error decreases 15--16-fold per halved step, consistent with its fourth-order integrator.
The same algorithm was applied to the 3-DoF series-parallel manipulator of~\cite[Sec.~IV]{ip} and the excavator record of~\cite{robocompiler}, each with two choices of independent coordinates. Across 192 additional feasible states, maximum discrepancies were $1.4\!\times\!10^{-9}$ in $\delta_a$ and $3.4\!\times\!10^{-10}$ in $\delta_b$, supporting transfer across mechanism classes.

\begin{table*}[!t]
\centering

\caption{Thirty-second dances: RMS position and force-prediction errors, minimum sole normal force, and allocation steps with slack above 0.1. Force statistics exclude the first 0.5~s. Dash: unavailable or inapplicable; ``fell'': fall time.}

\label{tab:dance}
\footnotesize\setlength{\tabcolsep}{3.6pt}
\begin{tabular}{@{}l rrrrrr rrr rr@{}}
\toprule
& \multicolumn{6}{c}{Ring, realistic plant} & \multicolumn{3}{c}{Ring, ideal plant} & \multicolumn{2}{c}{Realistic, pelvis} \\
\cmidrule(lr){2-7}\cmidrule(lr){8-10}\cmidrule(l){11-12}
Policy & pelvis & palm & link $F$ & sole $F$ & min $F_z$ & infeas. & pelvis & palm & link $F$ & semicircle & line \\
& mm & mm & err.\ N & err.\ N & N & \% & mm & mm & err.\ N & mm & mm \\
\midrule
Proposed & 0.73 & 1.92 & 2.2 & 30 & 55 & 0 & 0.029 & 0.080 & 0.01 & 0.72 & 0.72 \\
Projected tree~\eqref{eq:rc} & 0.73 & 1.92 & -- & -- & 55 & 0 & -- & -- & -- & -- & -- \\
Per-robot models & 0.76 & 2.00 & -- & 39 & 21 & 11 & 0.191 & 0.223 & -- & 0.71 & 0.72 \\
No curvature & 0.73 & 2.09 & 2.3 & 30 & 54 & 0 & 0.053 & 0.725 & 0.53 & 0.72 & 0.72 \\
No actuator inertia & 1.28 & 2.52 & 2.4 & 33 & 55 & 0 & 0.587 & 0.852 & 0.40 & 1.26 & 1.27 \\
Massless actuators & 1.47 & 2.47 & 2.4 & 37 & 62 & 0 & 0.944 & 0.822 & 0.44 & 1.45 & 1.46 \\
Gravity + PD & 1.01 & 6.20 & -- & -- & 52 & 0 & 0.836 & 5.668 & -- & 1.01 & 1.01 \\
Actuator PD & 13.43 & 20.44 & -- & -- & 0 & -- & -- & -- & -- & fell, 17\,s & fell, 17\,s \\
\bottomrule
\end{tabular}
\end{table*}

\begin{figure*}[!t]
\centering
\includegraphics[width=\textwidth]{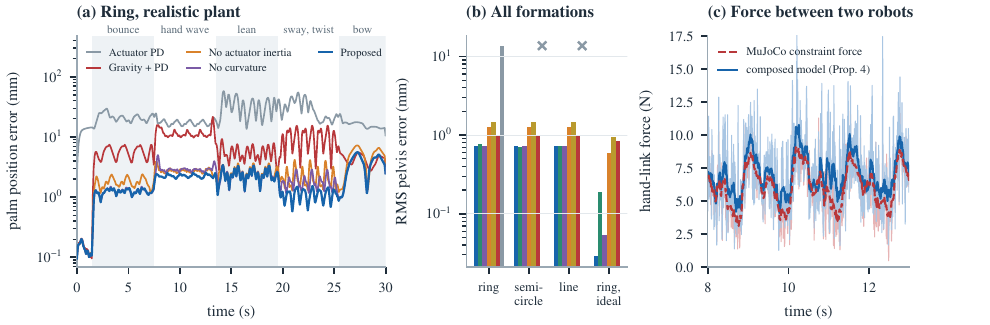}

\caption{Closed-loop dances. (a)~Largest palm error over ten palms, 0.2~s moving RMS; shading indicates choreography.
(b)~RMS pelvis error across formations and plants; colours follow (a), green denotes per-robot models, gold massless
actuators, and $\times$ a fall. (c)~Measured and predicted hand-link force magnitudes, 0.1~s moving means over raw signals.}

\label{fig:dance}
\end{figure*}

Table~\ref{tab:cost} reports median wall-clock model-evaluation times at one fixed feasible state per system (seed 2), after one untimed warm-up: 30 repetitions for the proposed model and Pinocchio, and five per projected-tree variant. Runs are sequential after other workflows, with one BLAS thread (Python 3.12.3, NumPy 2.3.5, Pinocchio 3.8.0, MuJoCo 3.15.0). Ring evaluation, including $C$ and analytic curvature, takes 3.5~ms versus 5.1~ms for projected-tree dynamics with analytic curvature, 9.2~ms with finite differences, 18.4~ms for rigid Pinocchio dynamics, and 0.21~ms for regularised MuJoCo forward dynamics. At the same robot state, ring and line $\Gamma^1_k$ agree to $1.8\!\times\!10^{-16}$ (relative); one- and two-level terms agree across six ring modes to $5.8\!\times\!10^{-16}$. Composer construction and preparation take 15.6~ms. A 3~kg payload requires a rank-six update taking 35~$\mu$s versus 4.0~ms for full reevaluation, agreeing to $3.3\!\times\!10^{-16}$ (relative). Recovered wrenches~\eqref{eq:wrench} agree with Pinocchio~\cite{proximal} within $3.3\!\times\!10^{-11}$~N.

Table~\ref{tab:dance} and Fig.~\ref{fig:dance} summarise tracking and forces. On the \emph{ideal plant}, the proposed ring model gives 0.029~mm RMS pelvis error and 0.080~mm palm error. Removing actuator inertia increases pelvis error 20.5-fold, omitting curvature increases palm error 9.1-fold, and per-robot models increase pelvis error 6.7-fold. On the \emph{realistic plant}, all formations complete with 0.72--0.73~mm RMS pelvis error and 1.92~mm ring palm error. Removing actuator inertia increases pelvis error 1.8-fold; omitting curvature increases palm error by 9\%. Actuator-PD runs fall in the semicircle and line. Per-robot models give 0.76~mm ring pelvis error, with
$\|z\|_\infty>0.1$ at 11\% of allocation steps and a minimum
sole normal force of 21~N, compared with 55~N for the composed model.
Measured hand-link forces are 7.0~N RMS with composed allocation versus 1.8~N with per-robot models. Predicted force magnitudes agree with MuJoCo to 2.2~N RMS on the realistic plant and 0.01~N on the ideal plant. Mechanically equivalent projected-tree dynamics give the same 0.73~mm ring pelvis error; the proposed model gives 0.011~mm on the rigid Pinocchio plant. Adding a 2~kg point payload at one palm at $t=1$~s and updating the forearm's inertial parameters reduces RMS palm error over the 30~s dance from 7.43 to 1.81~mm.
Travelling formations execute five 3.9~s side-step cycles: pelvis references shift over one foot while alternating swing feet rise 45~mm, advance 100~mm, and return to contact. The controller follows the reference schedule, switching among precompiled double-, left-, and right-support composers over the same five robot modules. Swing soles leave the support model 20~ms before reference lift-off and return 40~ms after reference landing. Each switch selects independent speeds from measured velocity and transfers the preceding actuator-effort vector for allocation. MuJoCo continues through switches, resolving lift-off and landing through unilateral contacts. Row selectors are rank-checked during preparation and fixed during execution.
In the ring, 20 support-mode switches produce a 32.1$^\circ$ turn (reference 32.1$^\circ$), 0.69~mm RMS pelvis error, 0.9~mm maximum landing error, and 1.8~mm maximum stance-foot slip (Fig.~\ref{fig:steps}). Landing error measures horizontal reference error 0.15~s after scheduled contact; stance slip measures horizontal displacement from each scheduled fully loaded stance interval's start. Semicircle and line RMS pelvis errors are 0.70 and 0.63~mm; doubling nominal sensor noise gives 0.63~mm for the stepping semicircle. For loaded soles ($F_z>50$~N), measured centre of pressure exceeds 95\% of a sole half-extent in 1.3\% of ring samples; tangential-to-normal force exceeds 95\% of the floor friction coefficient in 0.1\%. Per-robot models fall during the first cycle in every formation (ring: 4.7~s), with allocation slack above threshold at 43\% of pre-fall control steps.

\section{Conclusion}\label{sec:conclusion}
The proposed algorithm compiles composable dynamics from physical records, preserving coupled dynamics through analytic curvature and two-level composition. Independent references confirm numerical consistency, and five hand-linked humanoids demonstrate actuator-aware tracking, force allocation, payload adaptation, and support switching in MuJoCo. Reusable robot reductions, local inertial updates, and member regressors support changing connectivity and inertial parameter estimation.

\bibliographystyle{IEEEtran}
\bibliography{refere}

\end{document}